\documentclass{article}
\usepackage{graphicx} 
\usepackage{enumitem,amssymb,url}
\usepackage{amsmath,amsxtra,amssymb,amsfonts,latexsym, mathrsfs, dsfont,bm,enumitem, amsthm, mathtools}
\usepackage{enumitem}
\usepackage{color,soul}
\usepackage{cite}
\usepackage{hyperref}
\usepackage{geometry}[margin=1in]
\usepackage{subcaption}

\usepackage{algorithmic}
\usepackage{algorithm}

\newtheorem{theorem}{Theorem}[section]
\newtheorem{proposition}[theorem]{Proposition}

\newtheorem{corollary}[theorem]{Corollary}
\newtheorem{lemma}[theorem]{Lemma}

\theoremstyle{definition}

\theoremstyle{definition}

\title{Sketching the Heat Kernel: Using Gaussian Processes to Embed Data}
\author{Anna Gilbert and Kevin O'Neill}

\newcommand{\R}{\mathbb{R}} 
\newcommand{\E}{\mathbb{E}} 
\newcommand{\N}{\mathbb{N}} 

\newcommand{\eps}{\epsilon} 

\DeclareMathOperator{\diag}{diag}
\DeclareMathOperator{\diam}{diam}

\begin{document}

\maketitle

\begin{abstract}
    This paper introduces a novel, non-deterministic method for embedding data in low-dimensional Euclidean space based on computing realizations of a Gaussian process depending on the geometry of the data. This type of embedding first appeared in \cite{MR3827227} as a theoretical model for a generic manifold in high dimensions.
    
    In particular, we take the covariance function of the Gaussian process to be the heat kernel, and computing the embedding amounts to sketching a matrix representing the heat kernel. The Karhunen-Lo\`eve expansion reveals that the straight-line distances in the embedding approximate the diffusion distance in a probabilistic sense, avoiding the need for sharp cutoffs and maintaining some of the smaller-scale structure.
    
    Our method demonstrates further advantage in its robustness to outliers. We justify the approach with both theory and experiments.
\end{abstract}

\section{Introduction}

Recent success in the analysis of high-dimensional data is often attributed, at least in part, to the tendency of such data to have some sort of underlying low-dimensional structure. For example, a collection of photographs of the same object but taken at different angles will compose a small subset of the space of all digital images. To take advantage of this structure, it is often desirable to embed the original data in a relatively low-dimensional Euclidean space, allowing one to better understand this structure or more quickly run algorithms to analyze it (such as those for clustering).

In this paper, we introduce a novel method for embedding data in low-dimensional Euclidean spaces. While motivated by theoretical results on manifolds (the oft-assumed underlying structure of high-dimensional data), the method ultimately relies on a heat kernel and a notion of distance, suggesting it may bear fruit when applied to any of the large class of metric spaces on which a well-behaved heat kernel exists, such as weighted graphs (see \cite{MR1617040} and the related literature, for instance).

The method works by constructing a Gaussian process $f$ on the data. Let $f_1,\ldots,f_k$ denote independent realizations of $f$. Then, the embedding of the data into $\R^k$ is computed via the formula
\begin{equation}\label{eq:embedding formula}
    h^k(x)= \frac{1}{\sqrt{k}}(f_1(x),\ldots, f_k(x)).
\end{equation}
We call an embedding of the form \eqref{eq:embedding formula} a \emph{Gaussian process embedding}.

Gaussian process embeddings seem to have first appeared in \cite{MR3827227}. There, the authors show that the reach of the normalized Gaussian process embedding of a fixed manifold $M$ converges almost surely to a constant depending only on $M$ and the choice of Gaussian process as $k\to\infty$. The result may be interpreted as saying that for most manifolds embedded in high dimensions, the extrinsic geometry will not be an obstacle for analysis. The purpose of the Gaussian process embedding is to determine a notion of ``most manifolds'' in a way which can be analyzed using previous work on Gaussian processes on manifolds.

Our motivation lies in the followup work of \cite{krishnan2017intrinsic}, which established that the metric induced by such Gaussian process embeddings converges in $C^2$ to the original metric on $M$ almost surely as $k\to\infty$. For practical applications, we work with finite subsets of $\R^n$ in place of continuous manifolds, so our focus will be on the straight-line distance of the embeddings: $\|h^k(x)-h^k(y)\|_2$. We show in Propositions \ref{prop:embedding_distance_cts} and \ref{prop:embedding_distances_discrete} that this distance in the embedding approximates a diffusion distance on the original data.

We provide a brief summary of the diffusion maps embedding and related diffusion distance \cite{lafon_thesis,diffusion_maps}. Given a smooth, decaying kernel on a subset $X$ of $\R^n$, one may normalize large powers of this kernel to approximate the heat kernel on $X$. In particular, this process allows one to construct a heat kernel on a finite subset beginning with only an affinity matrix. The eigenfunctions of the heat kernel are then used to embed the data into a low-dimensional Euclidean space, where the straight-line distance between points is an approximation of the diffusion distance. The diffusion distance, while perhaps different from the original distance, is itself a valuable quantity, as it better reflects the connectivity of the data as a graph. Furthermore, the metric induced by the diffusion distance on a Riemannian manifold agrees with the original metric, but the diffusion distance is more robust than the geodesic distance when working with finite sets. See Subsections \ref{subsec:diffusion_maps} and \ref{subsec:approx_heat_kernel} for more.

While embedding with diffusion maps remains a useful tool, Gaussian process embeddings, in contrast, do not rely on a cutoff of which top eigenvalues to consider. Rather, as shown by the Karhunen-Lo\`eve expansion (Theorem \ref{thm:KL_exp}), they approximate the diffusion distance by combining all the eigenvectors/functions in each component, allowing one to preserve small-scale information which is neglected by diffusion maps. See Subsection \ref{subsec:challenges} for the example of a stretched flat torus. Experimentally, we find that Gaussian process embeddings also perform well with regards to robustness to outliers. They are also at least as easy to compute as diffusion maps; constructing a Gaussian process embedding according to a Gaussian process with covariance matrix equal to the heat kernel at time $2t$ amounts to sketching the heat kernel at time $t$. Specifics of computation are discussed more in Subsection \ref{subsec:computational_time}.

We now take a moment to discuss two issues of embeddability in Euclidean space and their relevance with regards to our method. The Whitney embedding theorem (as typically first proven in a graduate course) states that any compact $d$-dimensional manifold may be embedded into $\R^{2d+1}$ (see \cite{MR2954043}, for instance). As the authors of \cite{MR3827227} remark, \ref{eq:embedding formula} is almost surely an embedding when $k\ge 2d+1$. The Nash embedding theorem requires $k$ on the order of $d^2$ for an embedding into $\R^k$ to be isometric; the main result of \cite{krishnan2017intrinsic} shows this occurs for Gaussian process embeddings as $k\to\infty$.

The paper of Jones, Maggioni, and Schul \cite{JonesEtAl} shows that small neighborhoods of a $d$-dimensional compact Riemannian manifold may be parameterized by $d$ cleverly chosen eigenfunctions of the Laplace-Beltrami operator on the entire manifold. An article by Bates \cite{MR3256785} proves that the entire manifold may be embedded in $\R^N$ via the lowest $N$ eigenfunctions, where $N$ depends on $d$, a lower bound for injectivity radius, a lower bound for Ricci curvature, and a volume bound. By varying such quantities, one may require arbitrarily large $N$ to get an embedding (isometric or not).

An intermediate result in \cite{JonesEtAl} is that a compact, $d$-dimensional Riemannian manifold may be locally parameterized by its heat kernel evaluated at $d$ distinct and well-chosen points. Computing such a parameterization would require taking a subsampling of the columns of the heat kernel matrix. Viewing subsampling as a form of sketching offers parallels with our method.

The second issue regarding embedding involves data which cannot be embedded in a bi-Lipschitz manner in Euclidean space of any dimension, such as the Heisenberg group \cite{MR1402671}.
This may be considered another reason to use the diffusion distance, which still exists on non-Riemannian manifolds and yet may be realized in an embedding in Euclidean space.

In Section \ref{sec:background}, we formalize the theory of Gaussian process embeddings and use Dudley's inequality to determine a rate of convergence of the induced distances to the diffusion distance. We discuss in detail the Karhunen-Lo\`eve expansion for Gaussian fields and use it to analyze the differences between Gaussian process embeddings and diffusion maps, which are also formalized in this section.

Section \ref{sec:discretization} focuses on the discretization of the tools and results developed in Section \ref{sec:background}, taking us one step closer to constructing Gaussian process embeddings of finite datasets. 

Section \ref{sec:algorithm} states explicitly the algorithms used to compute Gaussian process embeddings, their runtimes, and how they relate to diffusion maps. We also develop a version of the algorithm in which sketching by a Gaussian matrix is replaced with sketching by a matrix of i.i.d. symmetric Bernoulli variables. 

Lastly, Section \ref{sec:experiments} includes a wide variety of experiments addressing comparisons with diffusion maps, with a particular emphasis on robustness to outliers and the dimension of Euclidean space required for an embedding.

\section{Continuous Theory of Embeddings}\label{sec:background}

\subsection{Diffusion Maps}\label{subsec:diffusion_maps}


Let $(M,g_M)$ be a compact, Riemannian manifold without boundary whose Laplace-Beltrami operator has eigenvalues $0=\lambda_0\le \lambda_1\le \lambda_2\le \ldots \nearrow\infty$ and a corresponding orthonormal basis of eigenfunctions $\{\varphi_i\}_{i=0}^\infty$. The heat kernel on $M$ at time $t$ may then be written as
\begin{equation*}
    k_t(x,y)=\sum_{i=0}^\infty e^{-\lambda_i t}\varphi_i(x)\varphi_i(y).
\end{equation*}
Thus, viewed as a linear transformation on $L^2(M)$ defined by $f(x)\mapsto \int_Mf(x)k_t(x,y)dx$, the heat kernel has an orthonormal basis of eigenfunctions $(e^{-\lambda_i t}\varphi_i(x))_{i=0}^\infty$ with corresponding eigenvalues $(e^{-\lambda_it})_{i=0}^\infty$.

Let $\ell^2$ denote the Hilbert space of real-valued, square-summable sequences $(a_i)_{i=1}^\infty$ with metric $d_{\ell^2}\left((a_i)_{i=1}^\infty,(b_i)_{i=1}^\infty\right)=\sum_{i=1}^\infty (a_i-b_i)^2$. The following result of Berard, Besson, and Gallot may be viewed as motivation for diffusion maps (analogous to our inspiration in \cite{krishnan2017intrinsic}).

\begin{theorem}\cite{BerardEtAl}\label{thm:BerardEtAl}
    Let $(M,g_M)$, $(\varphi_i)_{i=0}^\infty$, and $(\lambda_i)_{i=0}^\infty$ be as above. Then,
    \begin{equation}\label{eq:define_psi_t}
        \psi_t:x\mapsto \sqrt{2}(4\pi)^{n/2}t^{(n+2)/4}\left(e^{-\lambda_jt/2}\varphi_j(x)\right)_{j=1}^\infty
    \end{equation}
    is an embedding of $M$ into $\ell^2$ for all $t>0$. Furthermore, the pullback metric $\psi_t^* d_{\ell^2}$ is asymptotic to $g_M$ as $t\to0$.
\end{theorem}


Observe that the top eigenfunction $\varphi_0$ is dropped from \eqref{eq:define_psi_t} as it is trivial (equal to a nonzero constant function).

One may desire to simply define the diffusion maps embedding to be
\begin{equation}\label{eq:def_diffusion_maps_cts}
    x\mapsto (e^{-\lambda_1t}\varphi_1(x),\ldots,e^{-\lambda_kt}\varphi_k(x)),
\end{equation}
where the constant scaling factor from \eqref{eq:define_psi_t} is dropped and $t/2$ is replaced with $t$ in the exponent as this is a mere re-scaling of the time variable. However, in practice one may not have direct access to the heat kernel on a given set (especially if it is a finite sample from a manifold), so it is beneficial to derive a version for more arbitrary kernels.

Let $(X,\mu)$ be a measure space with a symmetric, positive semidefinite kernel $a:X\times X\to[0,\infty)$. The kernel $a$ will be chosen to be similar to the heat kernel, but for now we maintain generality as we will make two distinct choices for it.

If
\begin{equation*}
    \int_X\int_Xa(x,y)^2d\mu(x)d\mu(y)<\infty,
\end{equation*}
then the associated operator
\begin{equation*}
    f\mapsto \int_X a(x,y)f(x)dx
\end{equation*}
is compact and has the eigendecomposition
\begin{equation*}
    a(x,y)=\sum_{l\ge0}\lambda_l\phi_l(x)\phi_l(y),
\end{equation*}
where $\lambda_0\ge\lambda_1\ge\ldots\ge 0$. From this, we define
\begin{equation*}
    a^{(t)}(x,y)=\sum_{l\ge0}\lambda_l^t\phi_l(x)\phi_l(y).
\end{equation*}

We define the diffusion distance at time $t$, $D^{(t)}:X\times X\to\R$, by
\begin{equation*}
    D_a^{(t)}(x,y)^2=\int_X (a^{(t)}(x,u)-a^{(t)}(y,u))^2d\mu(u)=\sum_{l\ge0}\lambda_l^{2t}(\phi_l(x)-\phi_l(y))^2.
\end{equation*}

The diffusion distance is equal to the straight line ($\ell^2$) distance from the embedding
\begin{equation}\label{eq:ell^2_diffusion_embedding}
    x\mapsto (\lambda_l^t\phi_l)_{l\ge 0}.
\end{equation}

Thus, we define the diffusion maps embedding by
\begin{equation}\label{eq:def_diffusion_maps_truncate}
    x\mapsto (\lambda_0^t\phi_0(x),\ldots,\lambda_k^t\phi_k(x)).
\end{equation}

By truncating the embedding to lie in $\R^k$, we compute the following approximation of the diffusion distance
\begin{equation*}
    \tilde{D}^{(t)}(x,y)^2=\sum_{l=0}^k\lambda_l^t(\phi_l(x)-\phi_l(y))^2.
\end{equation*}

Observe that if $a^{(t)}(x,y)$ is the heat kernel at time $t$, then \eqref{eq:ell^2_diffusion_embedding} is the embedding from \eqref{eq:define_psi_t}. We discuss how one might obtain the kernel $a$ in a discretized setting in Subsection \ref{subsec:approx_heat_kernel}. One often desires to work with a kernel which is not necessarily symmetric but row-stochastic; however, we will focus on the symmetric kernels for better comparisons with Gaussian process embeddings. For full details, see \cite{diffusion_maps,coifman_marshall,lafon_thesis}.

\subsection{Challenges of Eigenfunction Embeddings}\label{subsec:challenges}

It is a slight misnomer to refer to \eqref{eq:def_diffusion_maps_cts} as an ``embedding'' since for any given $d$-dimensional manifold $M$ it is possible to choose $k$ small enough such that it is not even injective. For instance, one may take $k$ to be smaller than the (intrinsic) dimension of $M$ whenever the latter is greater than 1. However, one may always choose $k$ large enough to obtain an embedding.

That said, for the diffusion maps embedding one may not determine the appropriate $k$ merely as a function of $d$, as demonstrated by the following example. Let $M=S^1\times r S^1$, where $S^1$ is the unit circle, identified with the interval $[0,2\pi)$, and $r S^1$ is the circle of radius $r>1$, identified with the interval $[0,2r \pi)$. A well-known orthonormal basis of eigenfunctions of the Laplace-Beltrami operator $\Delta_M$ of $M$ consists of functions of the form

\begin{multline}
    \varphi^{cc}_{m_1,m_2}(x,y)=\cos(m_1x)\cos(m_2y/r),\varphi^{cs}_{m_1,m_2}(x,y)=\cos(m_1x)\sin(m_2y/r),\\ \varphi^{sc}_{m_1,m_2}(x,y)=\sin(m_1x)\cos(m_2y/r),\varphi^{ss}_{m_1,m_2}(x,y)=\sin(m_1x)\sin(m_2y/r).
\end{multline}
Here, $m_1,m_2$ are arbitrary nonnegative integers which are allowed to be zero in the case they lie inside a cosine expression, resulting in a nonzero constant function. The corresponding eigenvalues are $m_1^2+m_2^2/r^2$; thus, the lowest nonzero eigenvalues are (for non-integer $r$)
\begin{equation*}
    \frac{1}{r^2},\frac{1}{r^2},\ldots,\frac{\lfloor r \rfloor^2}{r^2}, \frac{\lfloor r \rfloor^2}{r^2}.
\end{equation*}
The eigenfunctions of $\Delta_M$ corresponding to these eigenvalues, and likewise the top nontrivial eigenfunctions of the heat kernel, are of the form
\begin{equation*}
    \varphi^{cc}_{0,m_2}(x,y)=\cos(m_2y/r),\varphi^{cs}_{0,m_2}(x,y)=\sin(m_2y/r), \hspace{.25 in} 1\le m_2\le \lfloor r \rfloor
\end{equation*}
and don't ``see'' the $x$-coordinate. As a result, the diffusion maps embedding in \eqref{eq:def_diffusion_maps_truncate} is not truly an embedding for $k\le 2\lfloor r \rfloor$, and this issue may be exacerbated by taking $r$ arbitrarily large.

The commonly used spectral embedding \cite{belkin2003laplacian}
\begin{equation*}
    x\mapsto(\varphi_1(x),\ldots,\varphi_k(x))
\end{equation*}
faces the same obstacle. In contrast, and as observed in \cite{MR3827227}, Gaussian process embeddings are almost surely embeddings whenever $k\ge 2d+1$

\subsection{Gaussian Processes and Riemannian Manifolds}

Given a parameter set $T$, a \emph{Gaussian process} is a collection of Gaussian random variables $f:=(f_t)_{t\in T}$. We use the notation $f_t$ and $f(t)$ interchangeably. A Gaussian process is \emph{zero-mean} if $\E f_t=0$ for all $t\in T$. A zero-mean Gaussian process is determined uniquely by its positive semi-definite covariance function
\begin{equation*}
    C(s,t):=\E[f(s)f(t)].
\end{equation*}

Each Gaussian process as above is associated with a \emph{canonical metric}
\begin{equation*}
    D_C(s,t):=\left(\E(f_s-f_t)^2\right)^{1/2}.
\end{equation*}
While actually a pseudometric in the sense that $D_C(s,t)=0$ need not imply $s=t$ (e.g., if $C\equiv0$), we will keep with the standard terminology.

A \emph{random process} is a collection of (not necessarily Gaussian) random variables $(f_t)_{t\in T}$; we may also refer to its covariance function (defined the same as above), though two different random processes may have the same covariance function.

Of great motivation to our work is the following result from \cite{krishnan2017intrinsic}. Let $(M,g_M)$ be a Riemannian manifold with a Gaussian process $f$ such that
\begin{equation}\label{eq:constistency_with_metric}
    g_M(X,Y)=[Y_yX_xC(x,y)]|_{x=y},
\end{equation}
where $X$ and $Y$ are vector fields with values $X_x,Y_x\in T_xM$, the tangent space of $M$. If \eqref{eq:constistency_with_metric} holds, then we say $g_M$ is the \emph{metric induced by} $f$. Given a Riemannian manifold $(M,g_M)$, there always exists a Gaussian process $f$ inducing $g_M$; this is a consequence of the Nash embedding theorem (as in the proof of Theorem 12.6.1 in \cite{adler2007random}). In fact, if $\iota:M\to\R^k$ is an isometric embedding, then given any isotropic kernel $k(x,y)$ (that is, depending only on $|x-y|$) one may choose $f$ by setting $C(x,y)=k(\iota(x),\iota(y))$.


Define a (random) map:
\begin{equation}\label{eq:def random embedding}
    h^k(x)=\frac{1}{\sqrt{k}}(f_1(x),\ldots,f_k(x)),
\end{equation}
where $f_1,\ldots,f_k$ are independent realizations of $f$. Let $g_E^k$ denote the metric induced on $h^k(M)$ by the Euclidean metric in $\R^k$.

\begin{theorem}\cite{krishnan2017intrinsic}\label{thm:Krishnan}
    Let $(M,g_M)$ be a connected, orientable, compact, $C^3$ Riemannian manifold and $f:M\to \R$ a zero-mean Gaussian process with a.s. $C^3$ sample paths inducing the metric $g_M$. Then, with probability one, the pullback metric $(h^k)^*g_E^k$ converges to $g_M$ in the $C^2$ topology.
\end{theorem}

The importance of $C^2$-convergence is that it implies the convergence of intrinsic functionals on $M$, such as volume and Lipschitz-Killing curvatures whose computations depend on derivatives of $g_M$. However, just the assumption that $M$ be $C^1$ would be enough to establish convergence of the pullback metric. This topic is discussed further in \cite{MR3827227}.

While motivated by Theorem \ref{thm:Krishnan}, we will not apply it directly or even discretize it too directly. In particular, we will pick the Gaussian process with covariance function equal to the heat kernel in order to obtain the diffusion distance, even if the heat kernel is not isotropic.

\subsection{Karhunen-Lo\`eve Expansion}

The following theorem is a fundamental result in the study of Gaussian processes which can be found in various forms throughout the literature. (See \cite{KL_book} or Chapter 3 of \cite{adler2007random}, for instance.)


\begin{theorem}[Karhunen-Lo\`eve Expansion]\label{thm:KL_exp}
    Let $T\subset\R^n$ be compact and $C(x,y)$ be a symmetric, positive semidefinite function on $T\times T$ such that the integral operator $g\mapsto \int_T C(x,y)g(y)dy$ is compact on $L^2(T)$. Let $\lambda_1\ge\lambda_2\ge\ldots$ be the eigenvalues of the integral operator with associated orthonormal eigenfunctions $\varphi_1,\varphi_2,\ldots$.

    Then, the Gaussian process $f=(f_t)_{t\in T}$ with covariance function $C(x,y)$ is of the form
    \begin{equation}\label{eq:KL_cts}
    f(x)=\sum_{i=1}^\infty \xi_i\sqrt{\lambda_i}\varphi_i(x),
\end{equation}
    where $\xi_1,\xi_2,\ldots$ are i.i.d. $N(0,1)$ random variables.
\end{theorem}

The convergence in \eqref{eq:KL_cts} is in the sense of $L^2(T)$ almost surely; however, one may show the convergence is uniform almost surely if $C(s,t)$ is Lipschitz. This will be enough to justify our operations with the expression. (We chose the Lipschitz condition for simplicity's sake; weaker conditions would also suffice.)

When dealing with finite data, one may not have a desirable notion of a geodesic distance or a Riemannian metric at a point; even if the data is assumed to have been sampled from a manifold, approximation of these quantities may be considered difficult.  However, we may easily derive estimates on the straight-line distance and an approximation of the Riemannian metric over distances from the Karhunen-Lo\`eve expansion. This approximation works by replacing $\langle X,Y\rangle$ ($X,Y$ lying in the tangent space of the manifold at a point $x$) with the Euclidean inner product of secant lines intersecting at $x$.

\begin{proposition}\label{prop:embedding_distance_cts}
    Let $T,f,C,\lambda_i$, and $\varphi_i$ be as in Theorem \ref{thm:KL_exp}. Fix $x,y,z\in T$ and $k\ge0$. Then,
\begin{equation}\label{eq:exp_Riemannian_metric}
    \E \langle h^k(y)-h^k(x),h^k(z)-h^k(x)\rangle= \sum_{i=1}^\infty \lambda_i (\varphi_i(y)-\varphi_i(x))(\varphi_i(z)-\varphi_i(x)).
\end{equation}
    
    In particular, \begin{equation}\label{eq:exp_straight_line_dist}
        D_C(x,y)^2=\E (h^1(x)-h^1(y))^2=\E (h^k(x)-h^k(y))^2=\sum_{i=1}^\infty \lambda_i (\varphi_i(x)-\varphi_i(y))^2=D_C^{(1/2)}(x,y)^2.
    \end{equation}
    More generally, $D_A^{(t)}=D_{A^{2t}}$.
\end{proposition}



\begin{proof}

    Observe that
    \begin{equation*}
        \E \langle h^k(y)-h^k(x),h^k(z)-h^k(x)\rangle=\frac{1}{k}\E \sum_{i=1}^k(f_i(y)-f_i(x))(f_i(z)-f_i(x)).
    \end{equation*}

Thus, it suffices to show that
\begin{equation*}
    \E (f(y)-f(x))(f(z)-f(x))=\sum_{i=1}^\infty \lambda_i (\varphi_i(x)-\varphi_i(y))(\varphi_i(x)-\varphi_i(z)).
\end{equation*}

By \eqref{eq:KL_cts} and the independence of the $\xi_i$,
\begin{align*}
    \E (f(y)-f(x))(f(z)-f(x))&=\E\left[\left(\sum_{i=1}^\infty \xi_i\sqrt{\lambda_i}(\varphi_i(y)-\varphi_i(x))\right)\left(\sum_{i=1}^\infty \xi_i\sqrt{\lambda_i}(\varphi_i(z)-\varphi_i(x))\right)\right]\\
    &=\E\left[\sum_{i,j}\xi_i\xi_j\sqrt{\lambda_i\lambda_j}(\varphi_i(y)-\varphi_i(x))(\varphi_j(z)-\varphi_j(x))\right]\\
    &=\E \sum_{i=1}^\infty \xi_i^2\lambda_i (\varphi_i(y)-\varphi_i(x))(\varphi_i(z)-\varphi_i(x)),
\end{align*}
which is equal to the right-hand-side of \eqref{eq:exp_Riemannian_metric} by the fact that $\E \xi_i^2=1$ for $\xi_i\sim N(0,1)$.

The second equality in \eqref{eq:exp_straight_line_dist} follows immediately as the special case of \eqref{eq:exp_Riemannian_metric} where $y=z$. The first and fourth follow trivially from the definitions of the distances.

\end{proof}


\subsection{Quantitative Convergence}

Given a pseudometric space $(T,d)$, let $\diam(T,d)$ denote its diameter, that is, $\sup_{s,t\in T}d(s,t)$. If $\epsilon>0$, then $\mathcal{N}(T,d;\epsilon)$ is the covering number of $T$, that is, the smallest number $N$ such that $T$ is contained in the union of $N$ balls of radius $\eps$.

\begin{theorem}\label{thm:application_of_Dudley}
    Let $f$ be a mean-zero Gaussian process on a compact set $T\subset\R^n$ with Lipschitz covariance function $C(x,y)$. Then, there exists a constant $C>0$ such that
    \begin{equation*}
        \E \sup_{x,y\in T}\left|\|h^k(x)-h^k(y)\|^2-D_C(x,y)^2\right|\le \frac{C \diam(T,D_C)}{\sqrt{k}}\int_0^{\diam(T,D_C)}\log(\mathcal{N}(T,D_C;\epsilon)+1)d\epsilon.
    \end{equation*}
\end{theorem}

We note that the factor of $\diam(T,D_C)$ on the right-hand-side provides the correct scaling with respect to the distance in that replacing the covariance function $C(x,y)$ by a constant multiple of itself $rC(x,y)$ results in both sides scaling by a factor of $r^2$.

To prove Theorem \ref{thm:application_of_Dudley}, we will rely on a couple basic properties of covering numbers and diameters stated below.

\begin{proposition}
    Let $(T,d)$ be a pseudometric space of finite diameter.
    \begin{enumerate}
        \item If $d'$ is a pseudometric on $T$ such that $d'(x,y)=Cd(x,y)$ for some fixed constant $C>0$ and all $x,y\in T$, then
        \begin{equation}\label{eq:diam_rescale}
            \diam(T,d')=C\diam(T,d)
        \end{equation}
    and       
        \begin{equation}\label{eq:covering_number_rescale}
            \mathcal{N}(T,d',\epsilon)=\mathcal{N}(T,d,\epsilon/C)
        \end{equation}
        for all $\epsilon>0$.
        \item Define a pseudometric $d''$ on $T\times T$ by $d''((x_1,y_1),(x_2,y_2))=\sqrt{d(x_1,x_2)^2+d(y_1,y_2)^2}$. Then,
        \begin{equation}\label{eq:diam_square}
            \diam(T\times T,d'')=\sqrt{2}\diam(T,d)
        \end{equation}
    and
        \begin{equation}\label{eq:covering_number_square}
            \mathcal{N}(T\times T,d'',\epsilon\sqrt{2})\le \mathcal{N}(T,d,\epsilon)^2.
        \end{equation}
    \end{enumerate}
\end{proposition}

\begin{proof}
    We may prove and \eqref{eq:diam_rescale} and \eqref{eq:covering_number_rescale} by simply observing that a ball of radius $\epsilon$ in the $d'$ pseudometric is a ball of radius $\epsilon/C$ in the original $d$ pseudometric.

    Choose $x,y\in T$ such that $d(x,y)\ge \diam(T,d)-\delta$, where $\delta>0$. Then,
    \begin{equation*}
        d''((x,x),(y,y))=\sqrt{2}d(x,y)\ge \sqrt{2}\diam(T,d)-\sqrt{2}\delta.
    \end{equation*}
    As $\delta>0$ is arbitrary, $\diam(T\times T,d'')\ge \sqrt{2}\diam(T,d)$.
    
    Similarly, if $(x_1,y_1),(x_2,y_2)\in T\times T$ are arbitrary, then $d(x_1,x_2),d(y_1,y_2)\le \diam(T,d)$ and
    \begin{equation*}
        d''((x_1,y_1),(x_2,y_2))\le \sqrt{\diam(T,d)^2+\diam(T,d)^2}=\sqrt{2}\diam(T,d).
    \end{equation*}
    Thus, $\diam(T\times T,d'')\le \sqrt{2}\diam(T,d)$ and $\diam(T\times T,d'')=\sqrt{2}\diam(T,d)$.

    Let $B_1,\ldots, B_N$ be a collection of balls of radius $\epsilon$ and centers $x_1,\ldots,x_N$ whose union is $T$, where $N=\mathcal{N}(T,d,\epsilon)$. We claim
    \begin{equation*}
        \bigcup_{1\le i,j,\le N}B_{ij}=T\times T,
    \end{equation*}
    where $B_{ij}$ is the ball of radius $\epsilon\sqrt{2}$ and center $(x_i,x_j)\in T\times T$ under the pseudometric $d''$.
    
    Let $(x,x')\in T\times T$. Then, there exists $1\le i,j\le N$ such that $x\in B_i$ and $x'\in B_j$. By definition of $d''$, $(x,x')\in B_{ij}$. There are $N^2$ such $B_{ij}$, proving our claim.
\end{proof}

We say a random variable $X$ is sub-gaussian if there exists $K>0$ such that
\begin{equation*}
    \E \exp(X^2/K^2)\le 2,
\end{equation*}
in which case $\|X\|_{\psi_2}$ is the smallest such $K$.

Similarly, $X$ is sub-exponential if there exists $K'>0$ such that
\begin{equation*}
    \E \exp(|X|/K')\le 2,
\end{equation*}
in which case $\|X\|_{\psi_1}$ is the smallest such $K'$.

We will use two basic results on sub-gaussian and sub-exponential random variables, as well as a corollary of Bernstein's inequality.

\begin{lemma}\label{lemma:sub_gauss_product}
    If $X$ and $Y$ are sub-gaussian random variables, then
    \begin{equation}\label{eq:sub_gauss_product}
        \|XY\|_{\psi_1}\le \|X\|_{\psi_2}\|Y\|_{\psi_2}.
    \end{equation}
\end{lemma}

\begin{lemma}\label{lemma:centering}
    There exists $C>0$ such that if $X$ is a sub-exponential random variable, then
    \begin{equation}\label{eq:eq:centering}
        \|X-\E X\|_{\psi_1}\le C\|X\|_{\psi_1}.
    \end{equation}
\end{lemma}

Lemmas \ref{lemma:sub_gauss_product} and \ref{lemma:centering} are Lemma 2.7.7 and Exercise 2.7.10 in \cite{vershynin}, respectively.

In order to handle sums of subexponential random variables, we will rely on the following form of Bernstein's inequality (see Corollary 2.8.3 in \cite{vershynin}).
\begin{theorem}[Bernstein's Inequality]\label{thm:bernstein}
    Let $X_1,\ldots,X_N$ be independent, mean-zero, sub-exponential random variables. Then, there exists an absolute constant $c>0$ such that for all $t\ge0$,
    \begin{equation*}
        P\left\{\frac{1}{N}\left|\sum_{i=1}^N X_i\right|\ge t\right\}\le 2 \exp\left[-c\min\left(\frac{t^2}{K^2},\frac{t}{K} \right)N\right].
    \end{equation*}
\end{theorem}

\begin{corollary}\label{cor:bernstein}
    Let $X_1,\ldots,X_N$ be i.i.d., mean-zero, sub-exponential random variables. Then,
    \begin{equation}\label{eq:cor_bernstein}
    \left\|\sum_{i=1}^NX_i\right\|_{\psi_1}\le \frac{CK}{\sqrt{N}},
    \end{equation}
    where $K=\|X_1\|_{\psi_1}=\cdots=\|X_N\|_{\psi_1}$.
\end{corollary}

\begin{proof}
    By Bernstein's inequality (Theorem \ref{thm:bernstein}, there exists $c>0$ such that
    \begin{equation}\label{eq:berstein_application}
        P\left\{\left|\sum_{i=1}^N X_i\right|\ge t\right\}\le 2 \exp\left[-c\min\left(\frac{t^2}{K^2},\frac{t}{K} \right)N\right].
    \end{equation}

    By Proposition 2.7.1 of \cite{vershynin}, an equivalent (up to a constant multiple) definition of $\|\sum_{i=1}^NX_i\|_{\psi_1}$ is the smallest number $K'$ such that
    \begin{equation}\label{eq:K'_subexp}
        P\left(\left|\sum_{i=1}^NX_i\right|\ge t\right)\le 2 \exp(-t/K')
    \end{equation}
    for all $t>0$. We claim \eqref{eq:K'_subexp} holds with
    \begin{equation}\label{eq:def_K'}
        K'=\frac{1}{\sqrt{c\log(2)}}\frac{K}{\sqrt{N}}.
    \end{equation}

    If $t\le K'\log(2)$, then $2\exp(-t/K')\ge1$, so \eqref{eq:K'_subexp} holds trivially. So suppose going forward that
    \begin{equation}\label{eq:t_ge}
        t\ge K'\log(2).
    \end{equation}

    If $t^2/K^2\ge t/K$, i.e., $t\ge K$, then by \eqref{eq:berstein_application}
    \begin{equation}
        P\left(\left|\sum_{i=1}^NX_i\right|\ge t\right)\le 2\exp(-cNt/K)\le 2\exp(-t/K')
    \end{equation}
    since $K'\ge \frac{K}{N}\frac{1}{\sqrt{c}\log(2)}\ge c\frac{K}{N}$, as $c^{3/2}\le \frac{1}{\sqrt{\log(2)}}$. (One may check in the proof of Bernstein's inequality that $c<1$, or simply observe that one may take $c<1$ regardless.)

    If $t^2/K^2\le t/K$, i.e., $t\le K$, then by \eqref{eq:berstein_application}
    \begin{equation}
        P\left(\left|\sum_{i=1}^NX_i\right|\ge t\right)\le 2\exp(-cNt^2/K^2).
    \end{equation}

We want to show that $cNt^2/K^2\ge t/K'$,or
\begin{equation*}
    cNt/K^2\ge 1/K'.
\end{equation*}

By \eqref{eq:t_ge},
\begin{equation*}
    cNt/K^2\ge cNK'\log(2)/K^2,
\end{equation*}
so it suffices to show that
\begin{equation*}
    cNK'\log(2)/K^2\ge1/K',
\end{equation*}
or
\begin{equation}\label{eq:final_equation}
    (K')^2\ge \frac{K^2}{cN\log(2)}.
\end{equation}
Equality holds in \eqref{eq:final_equation} by the definition of $K'$ in \ref{eq:def_K'}, so we are done. One obtains the arbitrary constant in \eqref{eq:cor_bernstein} by combining $\frac{1}{\sqrt{c\log(2)}}$ with the constant implicit in Proposition 2.7.1 of \cite{vershynin}.    
\end{proof}

\begin{theorem}[Dudley Inequality]\label{thm:Dudley}
    Let $(f_t)_{t\in T}$ be a mean-zero random process on a pseudometric space $(T,d)$ such that
    \begin{equation*}
        \|f_s-f_t\|_{\Psi_1}\le d(s,t) \text{ for all } s,t,\in T.
    \end{equation*}
    Then,
    \begin{equation*}
        \E \sup_{t\in T}|X_t|\le 8\int_0^{\diam(T,d)}\log(\mathcal{N}(T,d;\epsilon)+1)d\epsilon.
    \end{equation*}
\end{theorem}

Theorem \ref{thm:Dudley} follows immediately from the proof of Theorem 11.1 in \cite{chaining_textbook} by taking the random process to be mean-zero, and restricting to the special case of subexponential norms.


\begin{proof}[Proof of Theorem \ref{thm:application_of_Dudley}]
Recall that
\begin{equation*}
    h^k(x)=\frac{1}{\sqrt{k}}(f_1(x),\ldots,f_k(x)),
\end{equation*}
where each $f_i$ is an independent realization of the given Gaussian process $f$. 

    Define a random process $(X_{xy})_{(x,y)\in T\times T}$ by
    \begin{align*}
        X_{xy}&=\|h^k(x)-h^k(y)\|^2-d_C(x,y)^2\\
        &=\frac{(f_1(x)-f_1(y))^2+\ldots+(f_k(x)-f_k(y))^2}{k}-D_C(x,y)^2.
    \end{align*}

By \eqref{eq:KL_cts} and the independence of the $f_i$,
\begin{equation*}
    f_i(x)=\sum_{j\ge 0}\xi_{ij}\sqrt{\lambda_j}\varphi(x),
\end{equation*}
where the $\xi_{ij}$ ($1\le i\le k, j\ge0$) are i.i.d. $N(0,1)$ random variables.

Thus,
\begin{equation*}
    X_{st}=\frac{1}{k}\sum_{i=1}^k\left(\sum_{j\ge0}\xi_{ij}\sqrt{\lambda_j}(\varphi(x)-\varphi(y))\right)^2-D_C(x,y)^2
\end{equation*}
and, factoring the difference of squares,
\begin{align*}
    X_{xy}-X_{wz}&=\frac{1}{k}\sum_{i=1}^k\left(\sum_{j\ge0}\xi_{ij}\sqrt{\lambda_j}(\varphi(x)-\varphi(y))\right)^2-\left(\sum_{j\ge0}\xi_{ij}\sqrt{\lambda_j}(\varphi(w)-\varphi(z))\right)^2-(D_C(x,y)^2-D_C(w,z)^2)\\
    &=\frac{1}{k}\sum_{i=1}^k\left(\sum_{j\ge0}\xi_{ij}\sqrt{\lambda_j}(\varphi(x)-\varphi(y)-\varphi(w)+\varphi(z))\right)\left(\sum_{j\ge0}\xi_{ij}\sqrt{\lambda_j}(\varphi(x)-\varphi(y)+\varphi(w)-\varphi(z))\right)\\
    &-(D_C(x,y)^2-D_C(w,z)^2).
\end{align*}
    For each $1\le i\le k$,
    \begin{equation*}
        Y_{1i}:=\left(\sum_{j\ge0}\xi_{ij}\sqrt{\lambda_j}(\varphi(x)-\varphi(y)-\varphi(w)+\varphi(z))\right)\sim N\left(0,\sum_{j\ge0}\lambda_j(\varphi(x)-\varphi(y)-\varphi(w)+\varphi(z))^2\right)
    \end{equation*}
and
    \begin{equation*}
        Y_{2i}:=\left(\sum_{j\ge0}\xi_{ij}\sqrt{\lambda_j}(\varphi(x)-\varphi(y)+\varphi(w)-\varphi(z))\right)\sim N\left(0,\sum_{j\ge0}\lambda_j(\varphi(x)-\varphi(y)+\varphi(w)-\varphi(z))^2\right).
    \end{equation*}

By the inequality $(a-b)^2\le 2(a^2+b^2)$, we have
\begin{align*}
    \sum_{j\ge0}\lambda_j(\varphi(x)-\varphi(y)-\varphi(w)+\varphi(z))^2&=\sum_{j\ge0}\lambda_j[(\varphi(x)-\varphi(w))-(\varphi(y)-\varphi(z))]^2\\
    &\le 2\sum_{j\ge0}\lambda_j[(\varphi(x)-\varphi(w))^2+(\varphi(y)-\varphi(z))^2]\\
    &=2(D_C(x,w)^2+D_C(y,z)^2).
\end{align*}
Thus,
\begin{equation*}
    \|Y_{1i}\|_{\psi_2}\le \sqrt{2(D_C(x,w)^2+D_C(y,z)^2)}
\end{equation*}
and similarly,
\begin{equation*}
    \|Y_{2i}\|_{\psi_2}\le \sqrt{2(D_C(x,z)^2+D_C(y,w)^2)}.
\end{equation*}

As a result, by \eqref{eq:sub_gauss_product} and \eqref{eq:eq:centering},
\begin{align*}
    \|Y_{1i}Y_{2i}-\E Y_{1i}Y_{2i}\|_{\psi_1}&\le C\|Y_{1i}Y_{2i}\|_{\psi_1}\\
    &\le C\|Y_{1i}\|_{\psi_2}\|Y_{2i}\|_{\psi_2}\\
    &= C\sqrt{(D_C(x,w)^2+D_C(y,z)^2}\sqrt{(D_C(x,z)^2+D_C(y,w)^2}\\
    &\le C \diam(T,D_C) d_{T\times T}((x,y),(w,z)),
\end{align*}
    where
    \begin{equation*}
        d_{T\times T}((x,y),(w,z)):=\sqrt{(D_C(x,w)^2+D_C(y,z)^2}.
    \end{equation*}

By definition,
\begin{equation*}
    X_{xy}-X_{wz}=\frac{1}{k}\sum_{i=1}^k Y_{1i}Y_{2i}-\E Y_{1i}Y_{2i},
\end{equation*}
the average of $k$ independent, mean-zero, sub-exponential random variables. By Corollary \ref{cor:bernstein} and the above,
\begin{equation*}
    \|X_{xy}-X_{wz}\|_{\psi_1}\le \frac{C\diam(T,D_C)}{\sqrt{k}} d_{T\times T}((x,y),(w,z)).
\end{equation*}

Defining a pseudometric on $T\times T$ by
\begin{equation*}
    d^*((x,y),(w,z))=\frac{C\diam(T,D_C)}{\sqrt{k}} d_{T\times T}((x,y),(w,z))
\end{equation*}
we may apply Dudley's inequality and obtain
    \begin{equation*}
        \E \sup_{(x,y)\in T\times T}|X_{xy}|\le 8\int_0^{\diam(T\times T,d^*)}\log(\mathcal{N}(T\times T,d^*;\epsilon)+1)d\epsilon.
    \end{equation*}

For ease of notation, set $r=\frac{C\diam(T,D_C)}{\sqrt{k}}$.
    
By \eqref{eq:diam_rescale}, \eqref{eq:covering_number_rescale}, and  the change of variables $\epsilon\to\epsilon/r$,
\begin{align*}
    \E \sup_{(x,y)\in T\times T}|X_{xy}|&\le 8\int_0^{r \diam(T\times T,d_{T\times T})}\log(\mathcal{N}(T\times T,d_{T\times T};\epsilon/r)+1)d\epsilon\\
    &=8r\int_0^{\diam(T\times T,d_{T\times T})}\log(\mathcal{N}(T\times T,d_{T\times T};\epsilon)+1)d\epsilon.
\end{align*}

By \eqref{eq:diam_square}, \eqref{eq:covering_number_square}, and the change of variables $\epsilon\to\epsilon/\sqrt{2}$,
\begin{align*}
    \E \sup_{(x,y)\in T\times T}|X_{xy}|&\le 8r\int_0^{\sqrt{2}\diam(T,D_C)}\log\left((\mathcal{N}(T,D_C;\epsilon/\sqrt{2})+1)^2\right)d\epsilon\\
    &=16\frac{C \diam(T,D_C)}{\sqrt{k}}\int_0^{\sqrt{2}\diam(T,D_C)}\log(\mathcal{N}(T,D_C;\epsilon/\sqrt{2})+1)d\epsilon\\
    &=16\sqrt{2}\frac{C \diam(T,D_C)}{\sqrt{k}}\int_0^{\diam(T,D_C)}\log(\mathcal{N}(T,D_C;\epsilon)+1)d\epsilon.
\end{align*}

\end{proof}



\section{Discretization of Gaussian Process Embeddings}\label{sec:discretization}

Let $T=\{x_1,\ldots,x_N\}\subset \R^n$. Fix a symmetric, positive semidefinite, $N\times N$ matrix $A$ and let $\lambda_1,\ldots,\lambda_N$ denote the eigenvalues of $A$ with corresponding orthonormal basis of eigenvectors $v_1,\ldots,v_N$. $A$ has a unique symmetric, positive semidefinite square root we denote $\sqrt{A}$.

If $g\sim N(0,I_N)$, where $I_N$ is the $N\times N$ identity matrix, then
\begin{equation*}
    f:=\sqrt{A}g
\end{equation*}
defines a Gaussian process $f:=(f_i)_{1\le i\le N}$ on $T$ with covariance matrix $A$, where the $i$-th entry of $f$ is associated with $x_i$. In practice, $A$ will depend on the $x_i$ explicitly.

We denote the canonical metric of the Gaussian process by
\begin{equation*}
    d_A(x_i,x_j)=\E[(f_i-f_j)^2]^{1/2}.
\end{equation*}

To generate a Gaussian process embedding of $X$ into $\R^k$, we compute the $N\times k$ matrix
\begin{equation*}
    Y=\frac{1}{\sqrt{k}}\sqrt{A}G,
\end{equation*}
where $G$ is an $N\times k$ matrix of i.i.d. $N(0,1)$ entries, and the rows of $Y$, denoted $y_1,\ldots,y_N\in\R^k$ are the images of $x_1,\ldots,x_N$, respectively. 

The diffusion distance with respect to the matrix $A$ at time $t$ is given by
\begin{equation}\label{eq:diff_distance_discrete}
    d_A^{(t)}(x_i,x_j)^2=\sum_{i=k}^N\lambda_k^{2t}[(v_k)_i-(v_k)_j]^2=\|a_i-a_j\|^2,
\end{equation}
where $v_i$ denotes the $i$-th entry of the vector $v$ and $a_1,\ldots, a_N$ are the rows of $A$.


\begin{theorem}[Discrete Karhunen-Lo\`eve Expansion]\label{thm:KL_exp_discrete}
    Let $f=(f_i)_{1\le i\le N}$, $\lambda_1,\ldots,\lambda_N$, and $v_1,\ldots, v_N$ be as above. Then,
    \begin{equation}\label{eq:KL_discrete}
    f=\sum_{i=1}^\infty \xi_i\sqrt{\lambda_i}v_i,
\end{equation}
    where $\xi_1,\ldots,\xi_N$ are i.i.d. $N(0,1)$ random variables.
\end{theorem}

While Theorem \ref{thm:KL_exp_discrete} follows immediately from Theorem \ref{thm:KL_exp}, we note that one may prove it directly with a simple application of the diagonalization properties of symmetric matrices and the rotational invariance of the standard multivariate normal $g$.

\begin{proposition}\label{prop:embedding_distances_discrete}
    Let $T$, $A$, $Y$, $f=(f_i)_{1\le i\le N}$, $\lambda_1,\ldots,\lambda_N$, and $v_1,\ldots, v_N$ be as above. Fix $x_i,x_j,x_l\in T$. Then,
\begin{equation}\label{eq:exp_Riemannian_metric_discrete}
    \E \langle y_i-y_j,y_i-y_k\rangle= \sum_{p=1}^N \lambda_p [(v_p)_i-(v_p)_j)((v_p)_i-(v_p)_l].
\end{equation}
    
    In particular, \begin{equation}\label{eq:exp_straight_line_dist_discrete}
        d_A(x_i,x_j)^2=\E \|y_i-y_j\|^2= \sum_{p=1}^N \lambda_p [(v_p)_i-(v_p)_j)]^2=d_{A}^{1/2}(x_i,x_j)^2,
    \end{equation}
    More generally, $d_A^{(t)}=d_{A^{2t}}$ for all $t>0$.
\end{proposition}

While Proposition \ref{prop:embedding_distances_discrete} follows immediately from Proposition \ref{prop:embedding_distance_cts}, we note that it may also be proven easily from the elementary property that $\E\langle v,g\rangle\langle w,g\rangle=\langle v,w\rangle$ for $v,w\in\R^N$ and $g\sim N(0,I_N)$.

A discrete version of Theorem \ref{thm:application_of_Dudley} is as follows:
\begin{theorem}
    Let $T$, $A$, $f=(f_i)_{1\le i\le N}$, be as above. Then,
    \begin{equation*}
        \E \sup_{1\le i,j,\le N}\left|\|y_i-y_j\|^2-d_A(x_i,x_j)^2\right|\le \frac{C \diam(T,d_A)}{\sqrt{k}}\int_0^{\diam(T,d_A)}\log(\mathcal{N}(T,d_A;\epsilon)+1)d\epsilon.
    \end{equation*}
\end{theorem}

As data can be noisy, it would be nice to know that small perturbations in measurement result in small differences in Gaussian process embeddings. However, if one repeats the same Gaussian process embedding with the same covariance matrix (i.e., zero noise), then one may expect entirely different results. While the two embeddings may have similar behavior, it may be difficult to measure their similarities or differences.

Thus, we choose to compare two realizations of the Gaussian process embeddings which are computed by sketching with the same Gaussian matrix.

\begin{proposition}[Robustness of Gaussian process embeddings]
    Let $A^2,B^2$ be symmetric positive semidefinite $N\times N$ matrices which serve as covariance matrices for Gaussian process embeddings in $\R^k$ $Y=\frac{1}{\sqrt{k}}AG$ and $Z=\frac{1}{\sqrt{k}}BG$, respectively. 
    
    Let $\mu_1,\ldots,\mu_N$ denote the eigenvalues of $A-B$. Then,
    \begin{equation*}
        \sum_{i=1}^N \|y_i-z_i\|_2^2\sim \sum_{i=1}^N \mu_i^2 \frac{\chi_i^2(k)}{k},
    \end{equation*}
    where $\chi^2(k)$ is the chi-squared distribution with $k$ degrees of freedom and $\chi^2_1(k),\ldots,\chi^2_N(k)$ are independent copies of $\chi^2(k)$. In particular,
    \begin{equation*}
        \E\sum_{i=1}^N \|y_i-z_i\|_2^2=\|A-B\|_F^2
    \end{equation*}
    and
    \begin{equation*}
        \mathbb{V}\sum_{i=1}^N \|y_i-z_i\|_2^2=2\|A-B\|_F^2/k,
    \end{equation*}
    where $\|M\|_F$ refers to the Frobenius norm of the matrix $M$ and $\mathbb{V}X$ is the variance of a random variable $X$.
\end{proposition}

\begin{proof}
    Denote the columns of $G$ by $g_1,\ldots, g_k\in\R^N$ and let $a_i$ and $b_i$ be the $i$-th columns of $A$ and $B$, respectively. Then,
    \begin{align*}
        \sum_{i=1}^N \|y_i-z_i\|_2^2&=\frac{1}{k}\sum_{i=1}^N\sum_{j=1}^N(a_i\cdot g_j-b_i\cdot g_j)^2\\
        &=\frac{1}{k}\sum_{j=1}^N\sum_{i=1}^N[(a_i-b_i)\cdot g_j]^2\\
        &=\frac{1}{k}\sum_{j=1}^N\|(A-B)g_j\|^2\\
        &=\frac{1}{k}\sum_{j=1}^N \sum_{i=1}^N \mu_i^2\xi_{ij}^2,
    \end{align*}
    where $\xi_{ij}$ are i.i.d. $N(0,1)$ random variables. Again swapping the order of summation,
    \begin{align*}
        \sum_{i=1}^N \|y_i-z_i\|_2^2&=\frac{1}{k}\sum_{i=1}^N \sum_{j=1}^N \mu_i^2\xi_{ij}^2\\
        &=\sum_{i=1}^N \mu_i^2 \frac{\chi_i^2(k)}{k}.
    \end{align*}

\end{proof}

We are particularly interested in the case where $A$ and $B$ are approximations of the heat kernel on two finite sets, one of which may be a perturbation of the other. Towards this end, it is notable that there are results in the literature bounding $\|A-B\|_F$ for two normalized affinity matrices $A$ and $B$ formed from close sets of points. In particular, see \cite{BorisEtAl} for robustness of the bistochastically normalized affinity matrix with respect to heteroskedastic noise.

We also observe that the quantity $\|A-B\|_F^2$ is a discretization of global diffusion distance between the pairs of points associated with $A$ and $B$ in \cite{coifman_hirn}.

\section{Algorithm}\label{sec:algorithm}

\subsection{Approximating the Heat Kernel}\label{subsec:approx_heat_kernel}

Here, we attempt to approximate the heat kernel on a manifold from a finite sample. This approximate heat kernel will take the role of the kernel $a$ in Subsection \ref{subsec:diffusion_maps} in determining the diffusion maps embedding and diffusion distance.

We choose the heat kernel so that the straight-line distances in our embedding coincide with the commonly-used diffusion distances associated with said heat kernel. We could attempt to mimic Theorem \ref{thm:Krishnan} and choose an arbitrary isotropic kernel. However, in working with a finite set it is difficult to work with geodesic distances so we restrict our motivation to the results on straight-line distances.

Let $X=\{x_1,\ldots,x_N\}\subset\R^n$. Fix $\eps>0$ and define a kernel matrix $K_\eps$ with $(i,j)$-th entry 
\begin{equation*}
    k_\eps(x_i,x_j)=h(|x_i-x_j|/\eps),
\end{equation*}
where $h:\R\to\R$ is a smooth function with exponential decay. A common choice is $h(x)=e^{-|x|^2/\eps}$. Define a vector
\begin{equation*}
    q_\eps(x_i)=\sum_{j=1}^N k_\eps(x_i,x_j)
\end{equation*}
and replace the original kernel $k_\eps$ with
\begin{equation*}
    \tilde{k}_\eps(x_i,x_j)=\frac{k_\eps(x_i,x_j)}{q_\eps(x_i)q_\eps(x_j)}.
\end{equation*}

Now let
\begin{equation*}
    v_\eps^2(x_i)=\sum_{j=i}\tilde{k}_\eps(x_i,x_j)
\end{equation*}
and define
\begin{equation*}
    A_\eps(x_i,x_j)=\frac{1}{v_\eps(x_i)^2}\tilde{k}_\eps(x_i,x_j).
\end{equation*}

The matrix $A_\eps$ is row-stochastic and, when $X$ is sampled from a manifold $M\subset\R^D$, has been shown to approximate the heat kernel on $M$. This holds even when the sampling is nonuniform \cite{diffusion_maps}. However, unlike the heat kernel on a manifold, it is generally not symmetric.

We would prefer to use a symmetric kernel for the purpose of Gaussian process embeddings. (Sketching $A_\eps$ will give us a version of \eqref{eq:KL_discrete} with its left singular vectors, not its eigenvectors, which may not even be orthogonal.) Thus, we define the symmetrized kernel
\begin{equation*}
\tilde{A}_\eps(x_i,x_j)=\frac{\tilde{k}_\eps(x_i,x_j)}{v_\eps(x_i)v_\eps(x_j)}.
\end{equation*}

The construction of $\tilde{A}_\eps$ is summarized in Algorithm \ref{alg:symmetric_affinity_matrix}.

\begin{algorithm}
\caption{Symmetric Normalized Affinity Matrix}
\begin{algorithmic}\label{alg:symmetric_affinity_matrix}
\STATE \textbf{Input:} Finite subset $X=\{x_1,\ldots, x_N\}\subset\R^n$, scale $\epsilon>0$
\STATE Form affinity matrix $K$ with entries $K_{ij}=e^{-|x_1-x_j|^2/\eps}$.
\STATE Compute vector $q=(q_1,\ldots,q_N)$ of row sums $q_i=\sum_{j=1}^N K_{ij}$.
\STATE Compute $\tilde{K}=\diag(q^{-1})K \diag(q^{-1})$, where $q^{-1}=(1/q_1,\ldots,1/q_N)$.
\STATE Form vector $v=(v_1,\ldots,v_N)$ of row sums of $\tilde{K}$: $v_i=\sum_{j=1}^N \tilde{K}_{ij}$,
where $\tilde{K}_{ij}$ is the $(i,j)$-th entry of $\tilde{K}$.
\STATE \textbf{Return} $\tilde{A}=\diag(1/\sqrt{v})\tilde{K}\diag(1/\sqrt{v})$, where $1/\sqrt{v}=(1/\sqrt{v_1},\ldots,1/\sqrt{v_N})$.
\end{algorithmic}
\end{algorithm}

The matrix $\tilde{A}_\eps$ does not have the same provable approximations to the heat kernel as $A_\eps$ when $X$ is sampled with nonuniform density from a manifold. To obtain the best of both worlds-- a kernel which is symmetric and approximates the heat kernel-- we use the bistochastic normalization of $k_\eps$ \cite{coifman_marshall}.

Our goal is to find $d(x)$ such that
\begin{equation*}
    b_\eps(x_i,x_j)=\frac{k_{\eps}(x_i,x_j)}{d(x_i)d(x_j)}
\end{equation*}
is bistochastic.



Such a $d(x)$ may be found as follows. (See \cite{coifman_marshall}.) Given $v=(v_1,\ldots, v_n)\in\R^N$, let $diag(v)$ be the $N\times N$ diagonal matrix with diagonal entries $v_1,\ldots,v_N$.


We want to find a diagonal matrix $D$ such that
\begin{equation}\label{eq:Sinkhorn_desired}
    D^{-1}K_\eps D^{-1}\mathbf{1}=\mathbf{1},
\end{equation}
where $\mathbf{1}$ is the all-ones vector in $\R^N$.

Set $D_0=I_N$ and inductively define
\begin{equation*}
    D_{i+1}=\diag(K_\eps D^{-1}\mathbf{1}).
\end{equation*}

The limit
\begin{equation*}
    D=\lim_{i\to\infty} D_i^{1/2}D_{i+1}^{1/2}
\end{equation*}
solves \eqref{eq:Sinkhorn_desired}. Computationally, we terminate the construction of the $D_i$ when $D_i^{1/2}D_{i+1}^{1/2}$ solves \eqref{eq:Sinkhorn_desired} within some given tolerance. Following \cite{BorisEtAl}, we use Algorithm \ref{alg:bistochastic_affinity_matrix}. (In the pseudocode, vector operations such as division and square roots are componentwise.)

In \cite{coifman_marshall}, is it shown that choosing $b_\eps$ to be bistochastic with respect to $1/q_\eps$ leads to a provable approximation of the heat kernel when the $x_i$ are sampled from a manifold via a nonuniform distribution. However, this sacrifices our ability to compare with Gaussian process embeddings, so we avoid it, instead choosing to sample from uniform distributions in our experiments for fairest comparisons.

\begin{algorithm}
\caption{Bistochastic Normalized Affinity Matrix}
\begin{algorithmic} \label{alg:bistochastic_affinity_matrix}
\STATE \textbf{Input:} Finite subset $X=\{x_1,\ldots, x_N\}\subset\R^n$, scale $\epsilon>0$, tolerance $\delta>0$ (usually set to $10^{-8}$)
\STATE Form affinity matrix $K$ with entries $K_{ij}=e^{-|x_1-x_j|^2/\eps}$
\STATE Let $d_0=\diag(\mathbf{1})$, $d_1=K \diag(d_0)^{-1}\mathbf{1}$, and $d_2=K \diag(d_1)^{-1} \mathbf{1}$.
\STATE Set $i=0$.
\WHILE{$\|d_{i+2}/d_i-\mathbf{1}\|_\infty>\delta$}
    \STATE $i\to i+1$.
    \STATE Compute $d_{i+2}=K \diag(d_{i+1})^{-1} \mathbf{1}$.
\ENDWHILE
\STATE Let $d=d_{i+2}^{1/2}d_i^{1/2}$.
\STATE \textbf{Return} $B=\diag(d)^{-1}K \diag(d)^{-1}$.
\end{algorithmic}
\end{algorithm}

Despite the advantages of the bistochastic normalization, we will use the symmetric normalized kernel $\tilde{A}_\eps$ in our experiments as well, as it is simpler and often faster to compute.

\subsection{Diffusion Maps Algorithm}

Now that we have a couple of methods for approximating the heat kernel on a discrete subset of $\R^n$, we use them to formally compute the diffusion maps embedding found in \eqref{eq:def_diffusion_maps_truncate}. Observe that in place of taking the powers of the affinity matrix, we may instead take powers of its eigenvalues.

We make one modification to the formula in \eqref{eq:def_diffusion_maps_truncate}, dropping the use of the top eigenfunction, as for the heat kernel this is a constant function which does nothing to separate points. Upon discretizing, this holds for the bistochastic normalization (the all-ones vector is trivially the top eigenvector). While the top eigenvector for the symmetric normalization is $v_\eps$, differing entries reflect differences in density more than differences in location. While these differences do contribute to the diffusion distance, we run experiments where points are sampled from manifolds via the uniform distribution to minimize this contribution.

Alternatively, one could include the top eigenvector in the diffusion maps embedding with the symmetric normalized kernel. However, this would require going up to 3 dimensions to properly embed a circle rather than 2, so it appears more fair discard the top eigenvector.

\begin{algorithm}
\caption{Diffusion Maps Embedding}
\begin{algorithmic}\label{alg:diffusion_maps}
\STATE \textbf{Input:} Finite subset $X=\{x_1,\ldots, x_N\}\subset\R^n$, scale $\epsilon>0$, $k\in\N$, $t>0$, normalization, (Optional: tolerance $\delta>0$)
\IF{normalization=symmetric}
    \STATE Let $A$ be the output of Algorithm \ref{alg:symmetric_affinity_matrix} with inputs $X$ and $\eps$.
\ELSIF{normalization=bistochastic}
    \STATE {Let $A$ be the output of Algorithm \ref{alg:bistochastic_affinity_matrix} with inputs $X$, $\eps$, and $\delta$.}
\ENDIF
\STATE Compute $k+1$ highest eigenvalues $\lambda_0\ge\lambda_1\ge\ldots \ge\lambda_k$ of $A$ and corresponding eigenvectors $v_0,v_1,\ldots, v_k$.
\FOR {$j = 1,\ldots,N$}
    \STATE Compute
    \begin{equation*}
        y_j=(\lambda_1^t(v_1)_j,\ldots,\lambda_k^t(v_k)_j).
    \end{equation*}
\ENDFOR
\STATE \textbf{Return} $\{y_1,\ldots,y_N\}$.
\end{algorithmic}
\end{algorithm}

\subsection{Gaussian process embeddings Algorithm}

We have set up Gaussian process embeddings for discrete sets in full generality in Section \ref{sec:discretization}. At this point, we choose to compute them using the approximations of the heat kernel for finite subsets of $\R^n$.

Specifically, if $A$ is an $N\times N$ matrix representing the heat kernel at time $t$ on a set $X=\{x_1,\ldots,x_N\}\subset\R^n$, then we compute the Gaussian process embedding using $A^2$ (the heat kernel at time $2t$) as the covariance matrix. This allows us to simply sketch the matrix $A$ and not have to compute any square roots.

Here, unlike as in Algorithm \ref{alg:diffusion_maps}, we must directly take powers of the heat kernel matrix in order to approximate the heat kernel at large times.


On the surface, diffusion maps does not have this issue. At any time $t$, one takes the same top eigenvectors. The individual coordinates may be reweighed, but the basic topology of the embedding will remain the same. However, computing diffusion maps efficiently (with respect to time) may still require taking large powers. (See Subsection \ref{subsec:computational_time}.)

\begin{algorithm}
\caption{Gaussian process embedding}
\begin{algorithmic}\label{alg:random_embed}
\STATE \textbf{Input:} Finite subset $X=\{x_1,\ldots, x_N\}\subset\R^n$, scale $\epsilon>0$, $k\in\N$, integer $p>0$, normalization, (Optional: tolerance $\delta>0$)
\IF{normalization=symmetric}
    \STATE Let $A$ be the output of Algorithm \ref{alg:symmetric_affinity_matrix} with inputs $X$ and $\eps$.
\ELSIF{normalization=bistochastic}
    \STATE Let $A$ be the output of Algorithm \ref{alg:bistochastic_affinity_matrix} with inputs $X$, $\eps$, and $\delta$.
\ENDIF
\STATE Compute $Y=\frac{1}{\sqrt{k}}A^pG$, where $G$ is an $N\times k$ matrix with i.i.d. $N(0,1)$ entries.
\FOR {$j = 1,...,N$}
    \STATE Let $y_j$ equal the $j$-th row of $Y$.
\ENDFOR
\STATE \textbf{Return} $\{y_1,\ldots,y_N\}$.
\end{algorithmic}
\end{algorithm}

\subsection{Alternate Forms of Sketching and Non-Gaussian Processes}

Motivated by the use of sketching in Algorithm \ref{alg:random_embed}, we modify our algorithm by sketching by a matrix of symmetric Bernoulli random variables in place of $N(0,1)$ random variables. This generates a (non-Gaussian) random process with the same covariance matrix ($A^2$) as before, provided the symmetric Bernoulli random variables are pairwise independent.

The \textit{pairwise} independence, which shows up as a hypothesis in the following proposition, allows one to store random matrices more easily.

\begin{proposition}\label{prop:nonG_random_embed}
Let $A$ be an $N\times N$, symmetric, positive semidefinite matrix and $X$ be a mean-zero random variable satisfying $\E X^2=1$. Let $G$ be an $N\times k$ random matrix whose entries are pairwise independent and each distributed according to $X$.

Define a random matrix $Y=\frac{1}{\sqrt{k}}AG$. Then, for all $1\le i,j,l\le N$,
\begin{equation}\label{eq:prop_nonGauss_conclusion}
    \E (y_i-y_j)\cdot(y_i-y_l)=(a_i-a_j)\cdot(a_i-a_k),
\end{equation}
where $y_1,\ldots,y_N$ are the rows of $Y$ and $a_1,\ldots,a_N$ are the rows of $A$.
\end{proposition}

In other words, Proposition \ref{prop:embedding_distances_discrete} holds with random processes which are not necessarily Gaussian.


\begin{proof}
Here, we use the notation $M_{ij}$ to denote the $(i,j)$-th entry of the matrix $M$.

Expanding the left-hand-side of \eqref{eq:prop_nonGauss_conclusion}, we have
    \begin{align*}
        \E (y_i-y_j)\cdot(y_i-y_l)&=\frac{1}{k}\E\sum_{p=1}^k (Y_{ip}-Y_{jp})(Y_{ip}-Y_{jp})\\
        &=\frac{1}{k}\E\sum_{p=1}^k \left(\sum_{q=1}^NA_{iq}G_{qp}-\sum_{q=1}^NA_{jq}G_{qp}\right)\left(\sum_{q=1}^NA_{iq}G_{qp}-\sum_{q=1}^NA_{lq}G_{qp}\right)\\
        &= \frac{1}{k}\sum_{p=1}^k \E\left(\sum_{q=1}^N(A_{iq}-A_{jq})G_{qp}\right)\left(\sum_{r=1}^N(A_{ir}-A_{lr})G_{rp}\right)\\
        &=\frac{1}{k}\sum_{p=1}^k\sum_{q=1}^N(A_{iq}-A_{jq})(A_{iq}-A_{lq})\\
        &=\sum_{q=1}^N(A_{iq}-A_{jq})(A_{iq}-A_{lq}),
    \end{align*}
    where in the penultimate step we apply the pairwise independence of the $G_{lk}$ and the fact that each of them has variance 1.
\end{proof}

We apply Proposition \ref{prop:nonG_random_embed} by computing another form of Gaussian process embedding in Algorithm \ref{alg:re_pm1}, in which the sketching of Algorithm \ref{alg:random_embed} is done with a matrix of i.i.d. symmetric Bernoulli random variables.

\begin{algorithm}
\caption{Gaussian process embedding II (Symmetric Bernoulli Version)}
\begin{algorithmic}\label{alg:re_pm1}
\STATE \textbf{Input:} Finite subset $X=\{x_1,\ldots, x_N\}\subset\R^n$, scale $\epsilon>0$, $k,p\in\N$, normalization, (Optional: tolerance $\delta>0$)
\IF{normalization=symmetric}
    \STATE Let $A$ be the output of Algorithm \ref{alg:symmetric_affinity_matrix} with inputs $X$ and $\eps$.
\ELSIF{normalization=bistochastic}
    \STATE Let $A$ be the output of Algorithm \ref{alg:bistochastic_affinity_matrix} with inputs $X$, $\eps$, and $\delta$.
\ENDIF
\STATE Compute $Y=\frac{1}{\sqrt{k}}A^pG$, where $G$ is an $N\times k$ matrix with symmetric Bernoulli entries.
\FOR {$j = 1,...,N$}
    \STATE Let $y_j$ equal the $j$-th row of $Y$.
\ENDFOR
\STATE \textbf{Return} $\{y_1,\ldots,y_N\}$.
\end{algorithmic}
\end{algorithm}

\subsection{Computational Time}\label{subsec:computational_time}

One may check by inspection that Algorithm \ref{alg:symmetric_affinity_matrix} runs in time $O(N^2)$, where $N$ is the number of points. Algorithm \ref{alg:bistochastic_affinity_matrix} runs in time $O(IN^2)$, where $I$ is the total number of iterations in the while loop. While there are various results in the literature on the convergence rate of similar algorithms (for instance, \cite{sinkhorn1} and \cite{sinkhorn2}) we include Algorithm \ref{alg:symmetric_affinity_matrix} in large part for cases in which these bounds may be large.

To discuss the runtime of Algorithms \ref{alg:diffusion_maps} and \ref{alg:random_embed}, it may help to review some numerical linear algebra. Classical techniques may be used to compute the top $k$ eigenvectors and eigenvalues of an $N\times N$ matrix $A$ in time $O(N^3)$. More recent algorithms in randomized numerical linear algebra can do the same in time $O(kN^2)$. The basic idea is to first compute $AG$, where $G$ is an $N\times (k+p)$ matrix of i.i.d. $N(0,1)$ entries and apply classical techniques on the resulting $N\times (k+p)$ matrix which is smaller and requires less runtime. Here, $p$ is a small, fixed integer parameter. Strong precision is guaranteed with high probability, provided $A$ has a rapidly decaying spectrum.

In the case $A$ has a slowly decaying spectrum, one uses the power iteration method. That is, by observing that $A^m$ has a rapidly decaying spectrum for some large $m$, one can apply the above method on $A^m$, computing $A^mG$ one step at a time (first $AG$, then $A^2G=A(AG)$, etc.). This step runs in time $O(kmN^2)$, which is still faster than $O(N^3)$ for large $N$.

See \cite{martinsson2019randomized} for more on randomized methods in numerical linear algebra.

Now let $A$ refer to the output of either Algorithm \ref{alg:symmetric_affinity_matrix} or Algorithm \ref{alg:bistochastic_affinity_matrix}.

Let us first consider the case where one decides not to take large powers of $A$. If $A$ has a rapidly decaying spectrum, then the sketch in Algorithm \ref{alg:random_embed} requires an additional $O(kN^2)$ amount of work. In Algorithm \ref{alg:diffusion_maps}, one may compute the top eigenvalues and eigenvectors in $O(kN^2)$ steps using the randomized techniques described above.

Now suppose $A$ has a slowly decaying spectrum. For Algorithm \ref{alg:diffusion_maps}, one may either compute the top eigenvectors in $O(N^3)$ steps (via the slower classical methods), or apply the power iteration method and take $O(kmN^2)$ steps. In theory, one could run Algorithm \ref{alg:random_embed} in $O(kN^2)$ steps. However, the slowly decaying spectrum, combined with the Karhunen-Lo\`eve expansion, encourages one to take large powers of $A$ anyway so one would not do this in practice. (We will see in a moment this takes the same $O(kmN^2)$ steps as diffusion in this case.)

If instead, one already desires to compute powers of $A$, the above reasoning shows that both Algorithm \ref{alg:diffusion_maps} and \ref{alg:random_embed} will take $O(kmN^2)$ steps, where $m$ is the desired power. Here, we assume $A^m$ will have a rapidly decaying spectrum, even if $A$ does not.

We note that if one wants to conduct a multiscale analysis using either embedding, then many of the steps may be done simultaneously for all scales at once. For Algorithm \ref{alg:diffusion_maps}, this is somewhat trivial, as a different scale means a different rescaling of each individual coordinate. For Algorithm \ref{alg:random_embed}, this follows from the computation of $A^mG$ as $A(A(\cdots(AG)))$ as in the power iteration method.

\section{Experiments}\label{sec:experiments}

Except where otherwise specified, our experiments were run as follows. Fix a probability distribution $q(x)$ on a submanifold $M\subset\R^D$.


Fix parameters $N,n,p,k_{\min},k_{\max}\in\mathbb{N}$ and $\eps>0$. For each of $N$ trials, we sample $n$ points from $M$ via the distribution $q$ to obtain a set $X$. From this same sample of points, we run a subset of the following algorithms:
\begin{itemize}
    \item Diffusion maps embedding (Algorithm \ref{alg:diffusion_maps}) with the symmetric normalization (DMS)
    \item Diffusion maps embedding (Algorithm \ref{alg:diffusion_maps}) with the bistochastic normalization (DMB)
    \item Gaussian process embedding (Algorithm \ref{alg:random_embed}) with the symmetric normalization (GPS)
    \item Gaussian process embedding (Algorithm \ref{alg:random_embed}) with the bistochastic normalization (GPB)
    \item Gaussian process embedding with sketching by a symmetric Bernoulli random matrix (Algorithm \ref{alg:re_pm1}) with the symmetric normalization (GPSBS)
    \item Gaussian process embedding with sketching by a symmetric Bernoulli random matrix (Algorithm \ref{alg:re_pm1}) with the bistochastic normalization (GPSBB)
\end{itemize}

Each of the embeddings was computed with target dimensions $k_{\min},k_{\min}+1,\ldots,k_{\max}$. The set $X$ and the parameters remained the same across the choice of embedding and target dimension for fair comparison. Furthermore, whenever a Gaussian matrix was generated for the purpose of sketching, it was used for both instances of Algorithm \ref{alg:random_embed}. The analogous statement holds for symmetric Bernoulli matrices and Algorithm \ref{alg:re_pm1}.

From each embedding $f:X\to \R^k$, we computed the smallest $L$ such that
\begin{equation}\label{eq:define_biL}
    cd(x,y)\le \|f(x)-f(y)\|_2\le cLd(x,y)
\end{equation}
for all $x,y\in X$ and some $c>0$. This $L$ is like the usual biLipschitz constant of the map $f$; however, it adjusts for rescaling. That is, if there exists $\lambda>0$ such that $\|f(x)-f(y)\|=\lambda d(x,y)$ for all $x,y\in X$, then we would like to recover a value of 1, rather than $\lambda$. The optimal $L$ may be computed by dividing the maximum dilation by the minimum dilation, that is,
\begin{equation*}
    L=\frac{\max_{x\neq y}[\|f(x)-f(y)\|/d(x,y)]}{\min_{x\neq y}[\|f(x)-f(y)\|/d(x,y)]}.
\end{equation*}

Note that the distance $d$ above is unspecified. It will refer to either the diffusion distance appropriate for the particular embedding or the Euclidean distance on $\R^D$. In most cases, we will stick to the diffusion distance as in \eqref{eq:diff_distance_discrete}, as it is provably approximated by both diffusion maps and Gaussian process embeddings; our purpose is to measure how well. However, we note that since, for embedded Riemannian manifolds, the metric induced by the diffusion distance is the same as that induced by the Euclidean metric. Hence, the difference is often mote. 

The purpose of computing $L$ is to capture the effects of self-intersection for Gaussian process embeddings (see Figure \ref{fig:S^1_example}) and the challenges described in Subsection \ref{subsec:challenges} for diffusion maps (see Figures \ref{fig:torus_chart} and \ref{fig:Klein_chart}). Both phenomena involve two points which are far away being sent very close to each other in the embedding, resulting in a large $L$.

For each embedding, we compute $\log(L)$ (where $L$ is the optimal constant in \eqref{eq:define_biL} and the logarithm is the natural logarithm) and determine the mean and standard deviation of this value across all trials, fixing the method of embedding and the target dimension $k$. The mean is plotted against $k$ with errorbar equal to the standard deviation. Multiple methods are plotted on the same axes for comparison.

For most experiments, the symmetric normalization and bistochastic normalization for the heat kernel lead to very similar results. To avoid clutter, we plot both only when they are graphically distinguishable. When they are not, we plot the methods using the more commonly used symmetric normalized kernels.

\subsection{Manifolds}

To begin with, we embed $S^1=\{(x,y):x^2+y^2=1\}$, with parameters $N=200, n=300, p=8, k_{\min}=2,k_{\max}=8,$ and $\eps=0.25$. The distribution on $S^1$ is taken to be the uniform measure. Results are plotted in Figure \ref{fig:S^1_chart}.

We see that in the case of $S^1$, DMS is significantly more effective at approximating the diffusion distance than GPS. This phenomenon may be explained in large part by looking at the case $k=2$. The top 2 nontrivial eigenfunctions of the heat kernel on $S^1$ are $\sin(\theta)$ and $\cos(\theta)$, which are enough to recover the basic structure of $S^1$, even if more eigenfunctions are required to fully obtain the diffusion distance.

However, a random curve in the plane may be self-intersecting, collapsing two points together. In the discrete case simulated above, this corresponds to a very high value of $L$ (see Figure \ref{fig:S^1_example}). A random curve in $\R^3$ may not self-intersect, but in higher dimensions Gaussian process embeddings still lack the optimization of diffusion maps.

\begin{figure}[hbt!]
    \centering
    \begin{subfigure}[b]{0.45\columnwidth}
         \centering
         \includegraphics[width=\columnwidth]{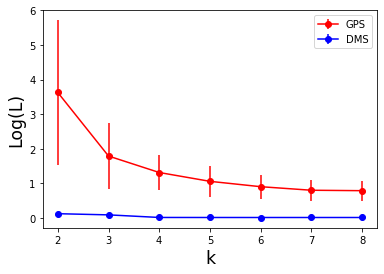}
         \caption{BiLipschitz constants for embeddings of $S^1$}
         \label{fig:S^1_chart}
     \end{subfigure}
     \hfill
     \begin{subfigure}[b]{0.45\columnwidth}
         \centering
         \includegraphics[width=\columnwidth]{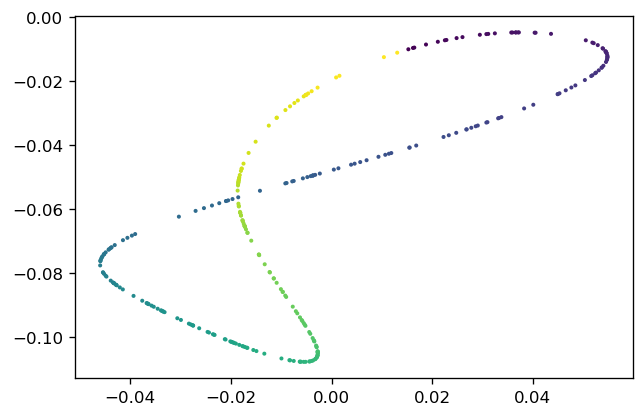}
         \caption{A sample Gaussian process embedding of $S^1$ into $\R^2$}
         \label{fig:S^1_example}
     \end{subfigure}
     \caption{Embeddings of $S^1$}
        \label{fig:S^1}
\end{figure}

Clearly, Gaussian process embeddings are not effective when the target dimension is $k=2$. However, a value of $L\approx 3$ in higher dimensions may be suitable for some practical applications.

To demonstrate the advantages of Gaussian process embeddings, we turn to our example from Subsection \ref{subsec:challenges}: $S^1\times r S^1$. Here, we take $N=100, n=500, p=10, k_{\min}=2, k_{\max}=12$, and $\eps=0.3$. 

We consider $S^1\times rS^1$ as a stretched version of the flat torus, originally living in $\R^4$. It may be described by the parameterized equations
\begin{equation*}
    (u,v)\mapsto (\cos(u),\sin(u),r\cos(v),r\sin(v)).
\end{equation*}
We also take $r=3.5$; here, theory predicts that the diffusion maps embedding will significantly struggle for $k\le 7$. Sampling on $S^1\times 3.5S^1$ is with respect to the uniform measure. Results are plotted in Figure \ref{fig:torus_chart}.

\begin{figure}[hbt!]
    \centering
    \begin{subfigure}[b]{0.45\columnwidth}
         \centering
         \includegraphics[width=\columnwidth]{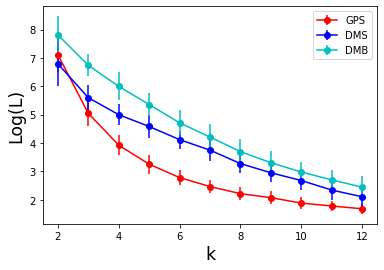}
         \caption{Embeddings of $S^1\times 3.5S^1$}
         \label{fig:torus_chart}
     \end{subfigure}
     \hfill
     \begin{subfigure}[b]{0.45\columnwidth}
         \centering
         \includegraphics[width=\columnwidth]{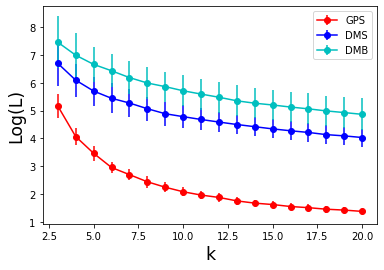}
         \caption{Embeddings of the Klein bottle}
         \label{fig:Klein_chart}
     \end{subfigure}
     \caption{Comparison of methods for two manifolds}
\end{figure}

In dimension $k=2$, the torus is not embeddable and all methods struggle somewhat equally. Beginning with dimension $k=3$ (when the torus becomes topologically embeddable), the Gaussian process embeddings start to distinguish themselves more strongly from the diffusion maps embeddings. Diffusion maps catches up in higher dimensions when the eigenfunctions which distinguish points of $S^1$ begin to be used.

An oft advantage of diffusion maps is that it denoises data by removing smaller-scale phenomena; however, the smaller-scale phenomenon here is the copy of $S^1$ in $S^1\times rS^1$, which should most likely not be considered noise due to its exact structure.

Lastly, we note that the diffusion maps embedding with the bistochastically normalized heat kernel fares worse than the one with the symmetric normalized kernel. While this would normally be considered a disadvantage, in this case it may be the reflection of a better estimator of the eigenfunctions of the torus.

Similar, yet more extreme phenomena occur for the Klein bottle, parameterized as
\begin{equation*}
(u,v)\mapsto\left((a+b\cos(v))\cos(u),(a+b\cos(v))\sin(u),b\sin(v)\cos(u/2),b\sin(v)\sin(u/2)\right).
\end{equation*}
Here, we take $a=10$ and $b=5$ and sample points with respect to the uniform measure on $(u,v)\in [0,2\pi)^2$.

For our parameters, we take $N=100, n=500, p=4, k_{\min}=3, k_{\max}=20$, and $\eps=2$. Results are plotted in Figure \ref{fig:Klein_chart}.

\subsection{Outliers}

For our next example, we sample 198 points from the uniform distribution on $S^1$ and adjoin the points $(0,3)$ and $(3,0)$ as outliers. Our parameters are $N=100, n=200, p=4, k_{\min}=2, k_{\max}=5$, and $\eps=0.5$. Results are depicted in Figure \ref{fig:S^1+outliers}.

\begin{figure}[hbt!]
    \centering
    \begin{subfigure}[t]{0.3\columnwidth}
         \centering
         \includegraphics[width=\columnwidth]{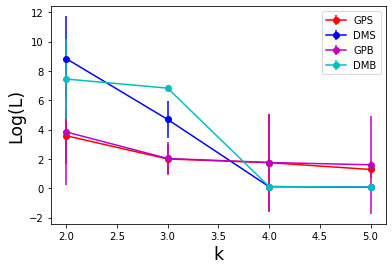}
         \caption{Embeddings of $S^1$ plus two outliers}
         \label{fig:S^1+outliers}
     \end{subfigure}
     \hfill
     \begin{subfigure}[t]{0.3\columnwidth}
         \centering
         \includegraphics[width=\columnwidth]{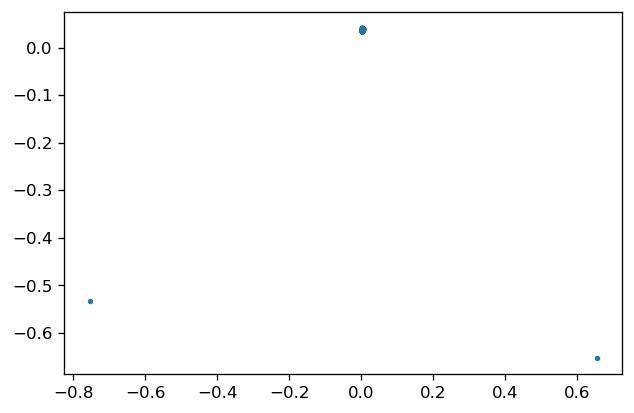}
         \caption{Sample embedding of $S^1$ plus two outliers with DMS}
         \label{fig:S^1+outDMS}
     \end{subfigure}
     \hfill
    \begin{subfigure}[t]{0.3\columnwidth}
         \centering
         \includegraphics[width=\columnwidth]{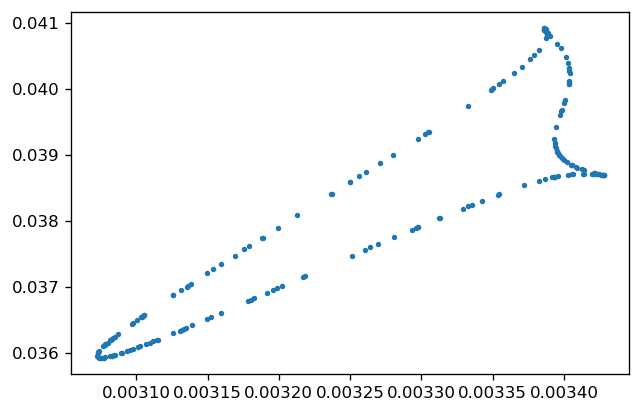}
         \caption{Sample embedding of $S^1$ plus two outliers with DMS, outliers removed}
         \label{fig:S^1+outDMS2}
     \end{subfigure}
        \caption{Embeddings with outliers}
\end{figure}

Here, the Gaussian process embedding methods significantly outperform the diffusion maps embeddings in dimensions $k=2$ and $k=3$, before the latter surpasses them for $k=4,5$. We include a sample embedding when $k=2$ for DMS in Figure \ref{fig:S^1+outDMS}. The two original outliers remain outliers; however, the 198 points from $S^1$ are all concentrated at the origin. Removing the outliers and zooming in, we obtain the image in Figure \ref{fig:S^1+outDMS2}. While some of the structure of $S^1$ remains, there are still issues with regard to self-intersection and the scale relative to the distance from the outliers.

DMS has to ``waste'' its top 2 eigenvectors on separating the outliers from the circle, leaving the circle to collapse on itself. Moving up to $k=3$ nearly removes the self-intersection problem, but one additional eigenvector is not enough to fully describe the circle.

Sample embeddings with $k=2$ for GPS may be found in Figures \ref{fig:RES1}, \ref{fig:RES2}, and \ref{fig:RES3}. While adding outliers does not remove the original self-intersection issue around embedding $S^1$ in $\R^2$, the images demonstrate a reasonable treatment of the outliers in the following sense. The scales of distances between points of $S^1$ and between $S^1$ and the outliers may be viewed simultaneously.

\begin{figure}[hbt!]
    \centering
    \begin{subfigure}[t]{0.3\columnwidth}
         \centering
         \includegraphics[width=\columnwidth]{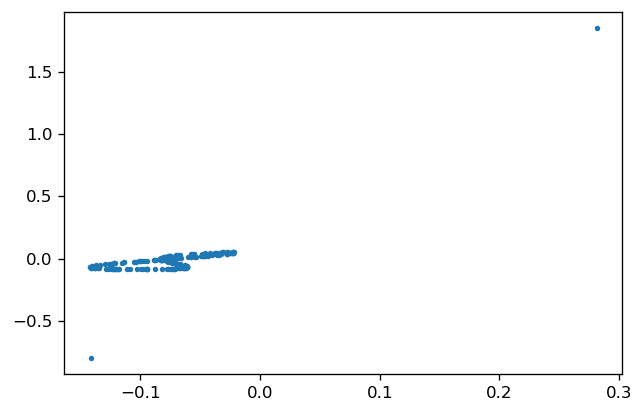}
         \label{fig:RES1}
     \end{subfigure}
     \hfill
     \begin{subfigure}[t]{0.3\columnwidth}
         \centering
         \includegraphics[width=\columnwidth]{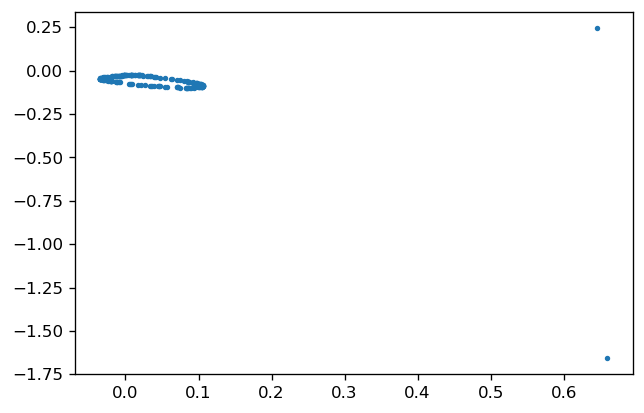}
         \label{fig:RES2}
     \end{subfigure}
     \hfill
    \begin{subfigure}[t]{0.3\columnwidth}
         \centering
         \includegraphics[width=\columnwidth]{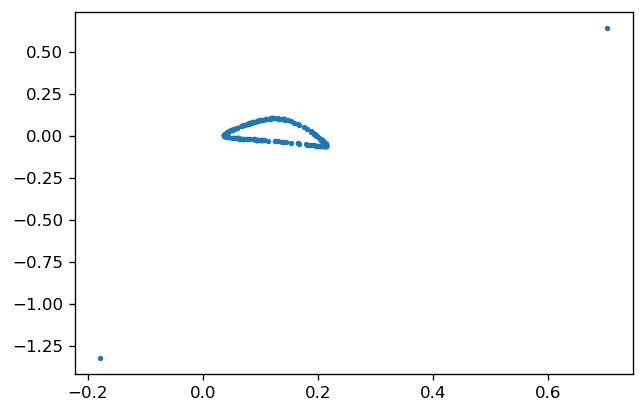}
         \label{fig:RES3}
     \end{subfigure}
        \caption{Three sample embeddings of $S^1$ plus two outliers with GPS}
\end{figure}

\subsection{Sketching with Symmetric Bernoulli Matrices}

We replot earlier results on the Klein bottle (Figure \ref{fig:Klein2}) and Gaussian point clouds (Figure \ref{fig:PC2}), this time for the GPS and GPSBS methods. Here, it appears that both methods, i.e., both forms of sketching, perform similarly. There does not appear to be any significant disadvantage to saving time or memory by replacing i.i.d. Gaussians with i.i.d symmetric Bernoulli random variables. In fact, when using the bistochastically normalized heat kernel there appears to be a slight advantage to using GPSBB (hidden in the figure by GPSBS) over GPB. However, more study is needed to determine if this difference extends to other examples.

\begin{figure}[hbt!]
    \centering
    \begin{subfigure}[b]{0.45\columnwidth}
         \centering
         \includegraphics[width=\columnwidth]{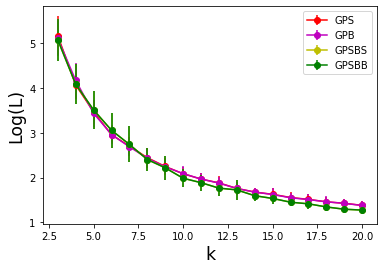}
         \caption{Results for the Klein bottle with all 4 types of Gaussian process embeddings}
         \label{fig:Klein2}
     \end{subfigure}
     \hfill
     \begin{subfigure}[b]{0.45\columnwidth}
         \centering
         \includegraphics[width=\columnwidth]{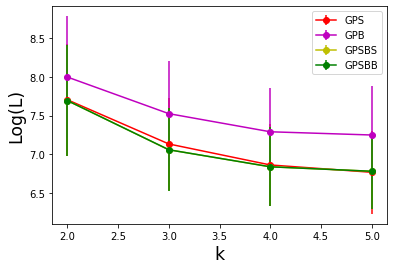}
         \caption{Results for point clouds with all 4 types of Gaussian process embeddings, comparison with Euclidean distance}
         \label{fig:PC2}
     \end{subfigure}
     \caption{Comparison of Gaussian process embeddings with i.i.d. Gaussian matrices and with symmetric Bernoulli matrices}
\end{figure}

\subsection{Multiscale Analysis}

While diffusion maps has a natural multiscale nature, taking powers of the heat kernel corresponds to taking powers of its eigenvalues while leaving the eigenvectors the same, resulting in a simple rescaling of the embedding.

For Gaussian process embeddings, the effect is more complicated. For extremely low powers of the heat matrix, many eigenvalues will be close to 1; by the Karhunen-Lo\`eve expansion, one will expect a random vector with distribution close to $N(0,I_N)$. For extremely large powers, the lower eigenvalues will disappear and one will expect a random vector nearly proportional to $rv_1$, where $v$ is the top eigenvector of the heat kernel and $r\sim N(0,1)$.

To test this effect we ran experiments the same as before except the target dimension $k$ was fixed while the power $p$ of the heat kernel took on the values $p=2,4,8,\ldots,2^P$ for some natural number $P$. Furthermore, the diffusion distance was replaced by the Euclidean distance in determining $L$.

Figure \ref{fig:power1} shows the results for $S^1\times 3.5S^1$ with $N=100,n=500,k=8,P=10$, and $\eps=.3$. As one can see, for low powers $p$, GPS struggles nearly as much as DMS before being a clear improvement for powers $p=16,32$. However, as $p$ increases, GPS gets progressively worse, eventually surpassing DMS.

Figure \ref{fig:power2} shows the results for $S^1$ with $N=100,n=300,k=2,P=8$, and $\eps=.25$. Similar behavior occurs as with the prior example, though DMS consistently functions better than GPS, as expected.

\begin{figure}[hbt!]
    \centering
    \begin{subfigure}[b]{0.45\columnwidth}
         \centering
         \includegraphics[width=\columnwidth]{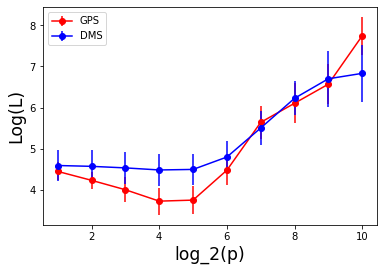}
         \caption{Embeddings of $S^1\times 3.5S^1$ into $\R^8$, $p$ varying, $L$ measured with Euclidean distance}
         \label{fig:power1}
     \end{subfigure}
     \hfill
     \begin{subfigure}[b]{0.45\columnwidth}
         \centering
         \includegraphics[width=\columnwidth]{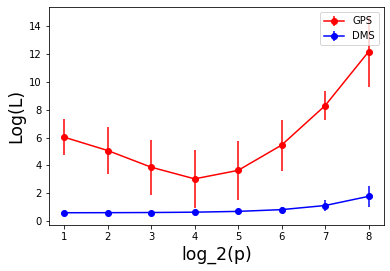}
         \caption{Embeddings of $S^1$ into $\R^2$, $p$ varying, $L$ measured with Euclidean distance}
         \label{fig:power2}
     \end{subfigure}
     \caption{A multiscale analysis of the embeddings}
\end{figure}

\noindent \textbf{Acknowledgments.} The authors would like to thank Ronald Coifman and Hanwen Zheng for productive discussions related to this work. We also thank Jeremy Hoskins for the title.


\bibliographystyle{plain}
\bibliography{RE.bib}

\end{document}